\documentclass{article}

\usepackage{microtype}
\usepackage{graphicx}
\usepackage{subcaption}
\usepackage{amsmath}
\usepackage{booktabs} 

\usepackage{hyperref}
\usepackage{xcolor}

\usepackage[accepted]{icml2026_weightsymmetry}

\usepackage{amsmath}
\usepackage{amssymb}
\usepackage{mathtools}
\usepackage{amsthm}
\usepackage{amssymb}
\usepackage{pifont}

\usepackage[capitalize,noabbrev]{cleveref}
\usepackage{dblfloatfix}

\newif\ifusesvg
\usesvgfalse    

\ifusesvg
  \usepackage{svg}
\fi

\newcommand{\myfigure}[2][]{%
  \ifusesvg
    \includesvg[#1]{#2.svg}%
  \else
    \includegraphics[#1]{#2.pdf}%
  \fi
}

\theoremstyle{plain}
\newtheorem{theorem}{Theorem}[section]

\theoremstyle{definition}

\theoremstyle{remark}

\usepackage{caption}

\usepackage[textsize=tiny]{todonotes}
\usepackage{enumitem}
\usepackage{subcaption}  

\icmltitlerunning{Objective-Specific Privileged Bases via Full-Prefix Matryoshka Learning}

\newcommand{\GL}{\mathrm{GL}}

\newcommand{\cM}{\mathcal{M}}
\newcommand{\cL}{\mathcal{L}}
\newcommand{\cG}{\mathcal{G}}
\newcommand{\RR}{\mathbb{R}}

\newcommand{\sgn}{\mathrm{sign}}

\newcommand{\LossAE}{\cL_{\mathrm{AE}}}
\newcommand{\LossMRL}{\cL_{\mathrm{MRL}}}

\newcommand{\LossFPMRL}{\cL_{\mathrm{FP\text{-}MRL}}}
\newcommand{\LossNUL}{\cL_{\mathrm{NU\text{-}\ell_2}}}
\newcommand{\LossMDL}{\cL_{\mathrm{MD\text{-}\ell_1}}}
\newcommand{\LossFisher}{\cL_{\mathrm{Fisher}}}
 \newcommand{\SBC}{S_{\mathrm{BC}}}
\newcommand{\SWC}{S_{\mathrm{WC}}}

\begin{document}

\twocolumn[
  \icmltitle{Objective-Specific Privileged Bases via Full-Prefix Matryoshka Learning}

\icmlsetsymbol{equal}{*}
  \begin{icmlauthorlist}
    \icmlauthor{Arghamitra Talukder}{columbia}
    \icmlauthor{Philippe Chlenski}{columbia}
    \icmlauthor{Itsik Pe'er}{columbia}
  \end{icmlauthorlist}

  \icmlaffiliation{columbia}{Computer Science, Columbia University}

  \icmlcorrespondingauthor{Arghamitra Talukder}{at3836@columbia.edu}

  \icmlkeywords{representation learning, symmetry breaking, Matryoshka,
    privileged basis, ordered representations, weight-space symmetry}
  \vskip 0.3in
]

\printAffiliationsAndNotice{}

\begin{abstract}
Learned representations are often invariant to rotational transformations, leaving individual dimensions non-identifiable and interchangeable. We study how Matryoshka Representation Learning (MRL) induces a task-aligned privileged basis
distinct from variance-based or regularizer-induced orderings. In the linear setting, we prove that full-prefix MRL recovers the
ordered principal directions, and can be computed
efficiently using shared statistics. Empirically, we demonstrate that MRL yields consistent per-dimension structure aligned with task signal, where coordinate magnitude reflects informativeness.

\end{abstract}

\section{Introduction}
\label{sec:intro}
Learned representations of high dimensional data are the backbone of modern machine
learning: once trained, lower dimensional embeddings can be reused
across retrieval, compression, and other tasks with
wildly different computational budgets.
A persistent obstacle is non-identifiability
where many embedding functions appear equally effective and the representational dimensions are not interpretable.
In particular, 
we focus on this occurring due to 
insensitivity of loss functions
to rotational transformations of 
the embedded data coordinates \citep{baldi1989neural}. 

Formally, this is orthogonal symmetry
of the $d$-dimensional embedding space
w.r.t. $\mathrm{SO}(d)$ transformations.
This symmetry needs to be broken in order to 
distinguish individual dimensions,
and to canonically order coordinate axes.
Moreover, we want the distinct dimensions to be meaningful in terms
of the training objective, so we extend the definition of a privileged
basis from \citet{elhage2023privileged} to an
\emph{objective-specific privileged basis}: a coordinate system whose
axes are not just unique and identifiable, but aligned with the
structure of the training objective. Axis ordering can arise from regularization or optimizer artifacts~\cite{elhage2023privileged}, and does not by itself constitute a privileged basis. Prescribing a mechanism for breaking the symmetry
in a meaningful way is thus a central question and the focus of this manuscript.

\textbf{Prior work.}
Prior methods answer this question in two structurally distinct ways. \emph{Nested prefix methods} 
--- nested dropout~\citep{rippel2014learning,bao2020regularized} and the 
Prefix-weighted reconstruction loss~\citep{oftadeh2020eliminating} optimize a weighted sum
of mean square error (MSE) losses across nested prefixes, recovering ordered
eigenvectors of the data covariance~$\Sigma$.
\emph{Regularizer-based methods} --- non-uniform $\ell_2$~\citep{kunin2019loss,bao2020regularized} 
and variance-ordering~\citep{zhan2026learning} impose ordering through monotone
penalty schedules on weight norms or latent variances, again
anchored to $\Sigma$. Both families have been studied primarily to optimize reconstruction loss, so the privileged
directions they recover are ultimately the PCA directions
of the data.

\textit{Matryoshka Representation Learning (MRL)}
\citep{kusupati2022matryoshka} reframes the prefix
reconstruction idea as a practical recipe for jointly training multiscale models with any differentiable loss $\mathcal{L}$,
and has since been widely adopted for efficiency-adaptive
deployment across retrieval, re-ranking, and vision-language
tasks \citep{xiao2026metaembed,cai2024matryoshka,hu2024matryoshka,verma2025matryoshka}.
Yet, despite this practical success, the structure of the privileged basis induced by MRL and its connection to PCA-aligned bases remains largely unexplored. 

To address this gap, we formalize, prove, and measure the objective-specificity of MRL's privileged
basis. Our contributions are:
 
\begin{enumerate}[label=\textbf{\arabic*.}, noitemsep, topsep=0pt, left=1em]
  \item \textbf{MRL's privileged basis is objective-specific.} We show that MRL's gradient weighting privileges dimensions in proportion to task utility, rather than reconstruction variance, and is structurally distinct from the sparsity patterns induced by regularizer-based methods.
  \item \textbf{Full-Prefix MRL recovers PCA for linear autoencoder at no additional asymptotic cost.} For linear autoencoder (LAE), the
  Full-Prefix MRL loss  recovers the ordered eigenvalue
  spectrum of the data covariance and exact eigenvectors under an
  orthonormality constraint at the same asymptotic cost as standard
reconstruction training, via shared matrix products. We extend the analysis to Fisher discriminant objective.
\end{enumerate}


\section{Theoretical Results}
\label{sec:theory}

\textbf{Setup.} A summary of all symbols used in this paper is provided in
Table~\ref{tab:notation}. Let $X \in \RR^{p \times n}$ be zero-centered data with empirical
covariance $\Sigma = XX^\top$, where $n$ is the number of samples
and $p$ the input dimension. An LAE has encoder $B \in \RR^{d \times p}$ and decoder $A \in \RR^{p \times d}$ with columns $a_1, \ldots, a_d$,
where $d \leq p$ is the latent dimension.
For a nesting set $\cM = \{m_1 < \cdots < m_K = d\} \subseteq [d]$, the latent code is $z = Bx \in \RR^d$, with $z_{1:m}$ denoting its first $m$ coordinates. 
We define the rank-1 contribution of dimension $k$ as $y_k = z_k a_k \in \RR^p$, so the $m$-th prefix reconstruction is $\hat{x}^{(m)} = \sum_{k=1}^m y_k$.

The standard \emph{LAE loss} is invariant under $(A,B)\mapsto(AT, T^{-1}B)$ for $T \in \mathrm{SO}(d)$ (orthogonal transformations), a subgroup of $\mathrm{GL}(d)$ (invertible linear transformations) \citep{baldi1989neural}. It is unordered and assigns equal weight to every dimension $k$; full expansion is given in~\ref{app:lae_expansion}.
\begin{equation}
  \label{eq:vanilla}
  \LossAE(\theta) = \Bigl\|x - \sum_{k=1}^d y_k\Bigr\|^2,
\end{equation}

\subsection{Objective-specific Privileged Basis}
\label{sec:Obj-spec}

\textbf{MRL loss.}
\citet{kusupati2022matryoshka} introduce Matryoshka Representation Learning
(MRL) loss that jointly trains representations across nested prefix
dimensions:
\begin{equation}
  \label{eq:mrl}
  \LossMRL(\theta) = \frac{1}{n}\sum_{s=1}^n
  \sum_{m \in \cM} \omega_m\,
  \cL\!\left(W^{(m)} z_{1:m}^{(s)};\, \tilde{y}_s\right),
\end{equation}
where $z_{1:m}^{(s)} \in \RR^m$ denotes the first $m$ coordinates of the
encoder output for sample $s$, $\omega_m \geq 0$ are per-scale loss weights
(we use $\omega_m = 1$ for all $m \in \cM$), and $W^{(m)} \in \RR^{L \times m}$ is a
linear head for scale $m$ that maps the prefix embedding to the
$L$-dimensional task output (e.g., $L = C$ for classification, $L = p$
for reconstruction), jointly trained with the encoder. We use the
weight-tied variant where, $W^{(m)} = W_{1:m}$ are the first $m$ columns of a
shared $W \in \RR^{L \times d}$. In practice, $m$ is chosen from a
geometrically spaced set, e.g. $\cM = \{2, 4, 8, \dots, 1024, 2048\}$, yielding
$O(\log d)$ nested prefixes; we refer to this spaced variant as
\emph{Sparse MRL (S-MRL)}.

\textbf{Full-prefix MRL loss for LAE.}
\begin{align}
  \label{eq:fullprefix_loss_struct}
  \LossFPMRL(\theta)
  =& \sum_{m=1}^d \omega_m \Bigl\|x - \sum_{k=1}^m y_k\Bigr\|^2\\
  =& \sum_{m=1}^d \omega_m\|x\|^2
  \nonumber
  - 2\sum_{k=1}^d w_k\,x^\top y_k\\
  &+ \mathbf{y}^\top(S_d^{\omega} \otimes I_p)\,\mathbf{y},
  \label{eq:expansion}
\end{align}
where $\omega_m > 0$ are per-prefix weights,
$w_k = \sum_{m=k}^d \omega_m$ is the cumulative weight on dimension $k$,
$\mathbf{y} = [y_1^\top,\ldots,y_d^\top]^\top \in \RR^{pd}$, and
$(S_d^{\omega})_{ij} = w_{\max(i,j)}$. Inducing a privileged basis requires  only that the cumulative
weights be strictly monotonic, $w_1 > w_2 > \cdots > w_d > 0$, which
holds for any positive choice of $\omega_m$. A milder schedule such as $\omega_m \propto 1/m$ or $\omega_m \propto 1/(d-m+1)$
flattens this ratio while preserving the monotonicity. We adopt $\omega_m = 1$ throughout
this work which recovers the standard form with $w_k = d-k+1$, where each dimension is
weighted by the number of prefixes it appears in. The full expansion is given in
\cref{app:FP-MRL_expansion}.



\textbf{Non-uniform regularization loss.} A natural approach to imposing dimension ordering is to penalize latent
dimensions with monotonically varying coefficients, forcing the model to
concentrate information in less-penalized coordinates. Two instantiations
of this idea bracket the design space.

Inspired by~\citet{kunin2019loss}, ~\citet{bao2020regularized} proposes a \emph{Non-uniform $\ell_2$
(NU-$\ell_2$)} regularization that penalizes the weights connected to each
dimension with strictly increasing $\ell_2$ coefficients $0 < \lambda_1 
\cdots < \lambda_d$. Heavier penalties on late dimensions force the LAE to explain
higher-variance directions of $X$ with early dimensions:
\begin{equation}
  \label{eq:nonuniform}
  \LossNUL(\theta) = \frac{1}{n}\|X - ABX\|_F^2
  + \|{\Lambda}^{1/2} B\|_F^2 + \|A {\Lambda}^{1/2}\|_F^2.
\end{equation}

We introduce as an exploratory baseline the \emph{Monotone-decay $\ell_1$
(MD-$\ell_1$)} loss, which instead penalizes the embedding mass of each
dimension with strictly increasing $\ell_1$ coefficients lighter on early
dimensions to encourage them to hold more information. To our knowledge, this specific monotone weighting
scheme has not been previously studied in this context:
\begin{equation}
  \LossMDL(\theta)
  = \frac{1}{n}\|X - ABX\|_F^2
  + \frac{\alpha}{n} \sum_{k=1}^{d} k \sum_{s=1}^{n} |b_k^\top x_s|,
\end{equation}
where $\alpha > 0$ controls the penalty strength and $b_k$ is the
$k$-th row of the encoder $B$. See appendix~\ref{app:NU-L2_expansion} and~\ref{app:MD-L1_expansion} for the full scalar derivation with
a concrete $p=5,\,d=3$ example.
\begin{figure*}[!t]
    \centering
    \myfigure[width=\textwidth]{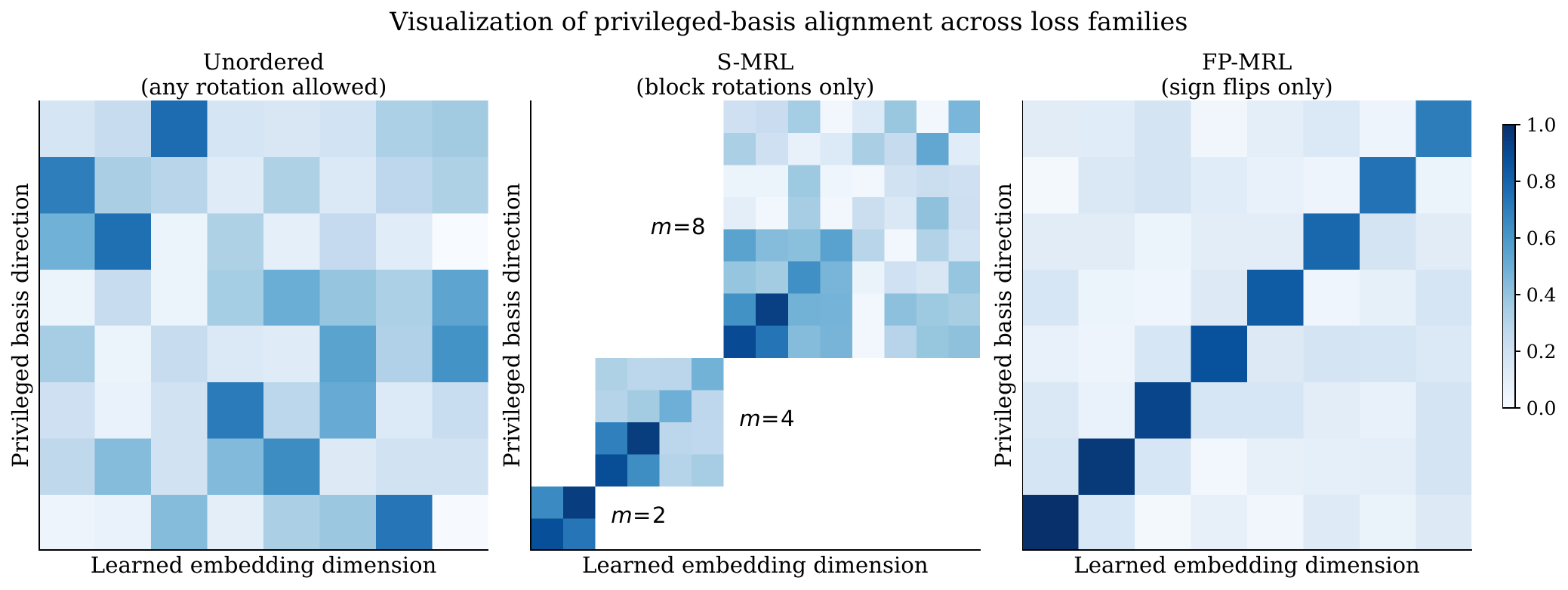}
    \caption{\textbf{Privileged-basis alignment visualization across loss families.}
Illustrative heatmaps showing the relationship between learned embedding
dimensions and a privileged reference basis under
different symmetry-breaking mechanisms. Unordered losses (left)
are invariant to arbitrary rotations of the latent basis, so any direction
can map to any dimension. S-MRL (middle) supervises only a
discrete set of prefix sizes ($m{=}2,4,8$), breaking the rotational
symmetry into block-diagonal structure with arbitrary rotations remaining
within each block. FP-MRL (right) supervises every prefix length,
collapsing the residual symmetry to sign flips and producing a
coordinate-wise identifiable basis.}
    \label{fig:heatmap}
\end{figure*}

Both losses share the same inductive bias: a monotone coefficient schedule
that makes early dimensions cheap and late dimensions costly, pushing
information toward low indices. However, the two penalties break rotational symmetry through
different mechanisms. Non-uniform $\ell_2$ breaks it through the
monotone weighting $\Lambda$: rotating the latent space mixes rows
of $B$ that carry different penalties, so any non-trivial rotation
strictly increases the regularizer. MD-$\ell_1$
breaks rotational symmetry at the level of the norm itself $\ell_1$ privileges the coordinate basis and the monotone
coefficient schedule additionally breaks permutation symmetry. Neither consults the task labels $y$ explicitly.

\textbf{Gradient comparison.}
The gradient of each loss with respect to $y_k$ reveals how it induces an ordering over the components of $y$: NU-$\ell_2$ scales by $\lambda_k(b_k \otimes a_k)$, ordering through
the weights alone; MD-$\ell_1$ by $k\cdot\sgn(b_k^\top x)\cdot x$,
ordering through the input. Neither uses the task signal. Full-prefix
MRL by $-2\sum_{m=k}^{d} r^{(m)}$, where
$r^{(m)} = x - \hat{x}^{(m)}$ is the prefix-$m$ residual, and the
ordering simultaneously acts on the gradient toward $x$
\emph{(ordered linear weights)}, the
self-penalty on $\|y_k\|^2$ \emph{(ordered quadratic weights)}, and the
coupling between dimensions (\emph{ordered cross-dim.\ coupling)} --- all
three shaped by the task loss through the learned encoder-decoder, see equation~\eqref{eq:expanded_full}. Unlike regularizer-based losses, which impose ordering on a single object (weights or activations), FP-MRL propagates task signal through all three
channels at once. NU-$\ell_2$ and MD-$\ell_1$ break symmetry by penalizing embedding geometry, tying privileged directions to data statistics or embedding mass. In contrast, FP-MRL accumulates task-gradient signal, inducing a privileged basis that is specific to the training objective--- the same architecture trained with a different
loss $\cL$ learns a different basis. 

Figure~\ref{fig:heatmap} provides an intuitive
visualization of how different symmetry-breaking mechanisms constrain the
learned latent basis. For unordered objectives, rotational symmetry remains
fully preserved, so privileged directions need not align with any coordinate
axis and the learned representation is effectively unidentifiable up to
arbitrary rotations. Sparse-prefix MRL (S-MRL) partially breaks this symmetry
by supervising only a discrete set of prefix lengths, producing block-wise
structure while still permitting arbitrary rotations within each block.
In contrast, FP-MRL supervises every prefix length, progressively constraining
the latent space until only sign ambiguities remain, yielding a
coordinate-wise identifiable privileged basis.



\subsection{PCA Recovery}
\label{sec:structure}

The ordered weighting of~\cref{eq:expansion} forces a nested chain of
principal subspaces. By the Eckart--Young--Mirsky Theorem, each prefix loss
$\|X - A_{1:m}B_{1:m}X\|_F^2$ is minimized iff
$\mathrm{colspan}(A_{1:m}) = \mathrm{span}(u_1,\ldots,u_m)$.
Standard LAE training enforces this only at the full width $m=d$,
which pins the column span of $A$ to the top-$d$ principal subspace
but leaves the choice of basis within that subspace free; FP-MRL imposes it at every
$m \in [d]$ simultaneously, and the unique family of subspaces
satisfying all $d$ constraints is the nested chain
$\mathrm{span}(u_1) \subset \cdots \subset
\mathrm{span}(u_1,\ldots,u_d)$. Equivalently, eigenvector $u_k$
contributes $\sigma_k^2$ to the bound of every prefix $m < k$, so the
summed loss weights $\sigma_k^2$ by $(k-1)$, and only the ordered chain
achieves the lower bound
$\sum_{k=2}^{d}(k-1)\sigma_k^2 + d\sum_{k=d+1}^{p}\sigma_k^2$. The resulting global optimum $(A^*, B^*)$ satisfies
$\mathrm{colspan}(A^*_{1:m}) = \mathrm{span}(u_1, \ldots, u_m)$ for all
$m \in [d]$, determined up to right-multiplication by $T \in \mathrm{SO}(m)$
reflecting the column-space invariance of MSE
\citep{baldi1989neural}. Adding $A^\top A = I_d$ restricts this gauge
to orthogonal transformations, and distinct eigenvalues then eliminate
the remaining rotational freedom at every scale, yielding the full PCA
recovery $A^* = [u_1, \ldots, u_d]$ \citep{rippel2014learning}.



\textbf{LDA recovery.} Appendix~\ref{app:fisher_loss} shows that the FP-MRL construction extends naturally beyond MSE reconstruction objectives to supervised discriminative objectives based on Fisher's criterion. In this setting, the encoder is trained to maximize between-class variance relative to within-class variance in the learned embedding, yielding the classical Linear Discriminant Analysis (LDA) directions as the optimal discriminative subspace. The role of the Eckart--Young--Mirsky theorem in the PCA setting is played here by the Rayleigh--Ritz characterization of the Fisher trace-ratio objective. Specifically, the Fisher objective is optimized by the generalized eigenvectors, where the corresponding eigenvalues measure class-discriminative strength through the Fisher ratio. Just as the MSE objective recovers PCA directions ordered by explained variance, the Fisher objective recovers LDA directions ordered by discriminative power. Thus, the privileged basis induced by FP-MRL adapts to the underlying training objective itself. We confirm the emergence of these ordered LDA directions empirically in~\cref{sec:experiments}. Unlike the MSE reconstruction objective, however, $\cL_{\mathrm{Fisher}}$ in equation~\eqref{eq:fisher} is non-convex. Consequently, optimization is substantially less stable in practice, exhibiting slower convergence. 

\begin{figure*}[t]
    \centering
    \myfigure[width=\textwidth]{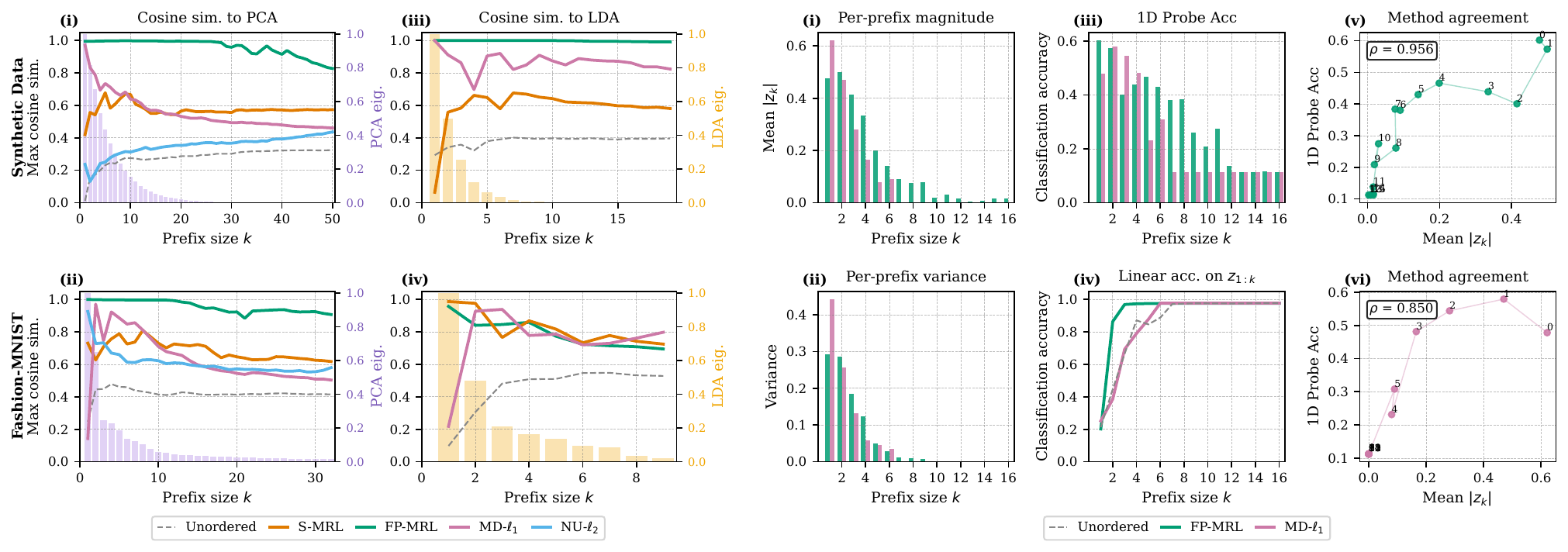}
    \begin{minipage}[t]{0.49\textwidth}
        \captionof{figure}{\textbf{Alignment with PCA and LDA bases under
        different symmetry-breaking mechanisms.}
        Maximum cosine similarity between learned prefix directions and
        the ordered PCA (i,~ii) and LDA (iii,~iv) bases, on synthetic data
        (top) and Fashion-MNIST (bottom). Bars show the normalized PCA/LDA eigenvalue
        spectra.}
        \label{fig:pcalda_alignment}
    \end{minipage}%
    \hfill
    \begin{minipage}[t]{0.49\textwidth}
        \captionof{figure}{\textbf{Per-prefix geometry and utility
        on MNIST classification. }Comparing FP-MRL and MD-$\ell_1$.
        (i,~ii) Mean $|z_k|$ and variance per coordinate.
        (iii) Top-1 probe accuracy of each coordinate in isolation.
        (iv) Cumulative top-1 accuracy on the prefix $z_{1:k}$.
        (v,~vi) Rank agreement between mean $|z_k|$ and 1D probe accuracy.}
        \label{fig:clf_geometry}
    \end{minipage}
\end{figure*}

\textbf{Efficient loss calculation.}
Naively, full-prefix MRL requires $d$ reconstruction losses, one per
prefix suggesting a $d$-fold cost increase over standard training.
In the linear setting, however, all prefix losses share the same
sufficient statistics and can be assembled from $d \times d$ matrices
computed once from $X$, $A$, and $B$. Full-prefix evaluation therefore
contributes only an $O(d^2)$ additive overhead on top of the $O(npd)$
matrix multiplications already required for standard reconstruction,
with no multiplicative blow-up. We note that for PCA itself a direct
SVD remains the most efficient route. This result is universally applicable across linear
models: even when no closed-form analytical solution exists (unlike
PCA or LDA), the same sufficient-statistic structure ensures that
full-prefix supervision incurs only additive overhead. A complete proof
and derivation of the computational complexity is given in
appendix~\ref{app:pca-proof}.

\textbf{Beyond MSE.}
The structural basis for the gradient ordering argument above and the efficient loss evaluation result in appendix~\ref{app:pca-proof} requires only that the model's prefix-$m$ output factorizes as a sum of per-coordinate contributions: besides a rank-1 contribution to reconstruction, this could be a vector contribution to logits or a scalar contribution to a pairwise distance.
The same gradient accumulation and efficient computation arguments extend naturally to classification, regression, and distance-based losses.
We leave the unified treatment for future work.

\section{Empirical Validation}
\label{sec:experiments}

 We evaluate PCA and LDA recovery on synthetic data, and on Fashion-MNIST in \cref{fig:pcalda_alignment}. \cref{fig:clf_geometry} evaluates per-dimension statistics and probe accuracy on MNIST. The model architecture, synthetic data generation, and evaluation metric are detailed in appendix~\cref{app:exp_setup}, with hyperparameters summarized in \cref{tab:arch_dims}.


\textbf{Alignment with PCA and LDA.}
\cref{fig:pcalda_alignment}(i--ii) shows the maximum cosine similarity between learned prefix directions and PCA eigenvectors.
FP-MRL maintains near-unit alignment across most prefix sizes on both synthetic and Fashion-MNIST data, with slight degradation at larger $k$. This drop occurs as later eigenvalues become small and closely spaced: when $\sigma_k^2 \approx \sigma_{k+1}^2$, multiple directions explain nearly identical variance making alignment ill-conditioned. In contrast, S-MRL, NU-$\ell_2$, and MD-$\ell_1$ either degrade with prefix size or fail to align consistently.

When trained with a  Fisher discriminant objective (appendix~\ref{app:fisher_loss}), we observe the same qualitative behavior for LDA directions~\cref{fig:pcalda_alignment}(iii--iv).
On synthetic data, FP-MRL achieves near-perfect alignment across prefix sizes, mirroring the PCA case. On real data, alignment degrades for all methods.
We attribute this to optimization rather than objective mismatch: the simple linear architecture described in appendix~\cref{app:arch_dims} used here does not reliably reach the global optimum of the LDA objective, leading to imperfect recovery of discriminative directions (~\cref{fig:losses}).
This gap suggests an important direction for future work—whether MRL requires more expressive architectures or optimization schemes to fully realize its theoretical guarantees under supervised objectives. Regularizer-based methods show weaker and less stable alignment throughout, reinforcing that their induced bases are not driven by task signal.


\textbf{Objective Specificity.}
To study whether MRL's induced ordering reflects the training objective itself rather than merely imposing sparsity, we move beyond the linear autoencoder setting to a nonlinear classification task on MNIST. We train a shared nonlinear encoder-classifier architecture on MNIST using either FP-MRL or MD-$\ell_1$ as the ordering objective. The model consists of a two-hidden-layer ReLU MLP encoder with batch normalization and dropout, followed by a simple linear classifier. We then evaluate whether the resulting coordinate ordering aligns with class-discriminative usefulness. Concretely, for each latent coordinate $z_k$, we measure both its magnitude statistics and its ability to linearly separate classes in isolation via a 1D logistic probe. We additionally evaluate prefix-wise linear classification accuracy using only the first $k$ coordinates. Full architectural, training, and evaluation details are provided in appendix~\cref{app:arch_dims}.

\cref{fig:clf_geometry} isolates how each training objective shapes the geometry of the
learned code along the prefix axis~$k$. Panels (i)--(ii) show that both
FP-MRL and MD-$\ell_1$ produce the
expected monotonic ordering of
the per-dimension magnitude $\mathbb{E}\!\left[|z_k|\right]$ with~$k$, concentrating
energy in the leading coordinates. The two objectives, however, induce
different decay profiles---MD-$\ell_1$ collapses mass more aggressively
into the first $2$--$3$ dimensions, whereas FP-MRL spreads usable signal
out to roughly $k\!\approx\!8$ before the tail goes inert. Subplot~(iii) shows that this geometric ordering is functionally
meaningful. The 1D probe accuracy measures how well a single coordinate $z_k$, used in
isolation, can linearly separate the classes, a direct test of whether
magnitude actually reflects per-dimension informativeness. MD-$\ell_1$ allocates the largest magnitude and variance to $k\!=\!1$, yet FP-MRL achieves the higher 1D probe accuracy there,
ties at $k\!=\!2$, and is only overtaken at $k\!=\!3$. Past $k\!=\!3$, MD-$\ell_1$'s probe accuracy decays in step with its
magnitude, whereas FP-MRL retains measurable signal across the remaining
coordinates. Subplot~(iv) reinforces this
pattern at the cumulative level: the linear accuracy on the prefix
$z_{1:k}$ rises sharply for FP-MRL and saturates near the
full-dimensional ceiling by $k\!\approx\!3$, while MD-$\ell_1$ requires
roughly twice as many coordinates ($k\!\approx\!6$) to reach the same
plateau. In deployment terms, this means FP-MRL can be truncated to half the
embedding width of MD-$\ell_1$ at no accuracy cost, a practical consequence of aligning magnitude
with informativeness. The MD-$\ell_1$ penalty
inflates the leading coordinate's scale as a byproduct of its sparsity
geometry, whereas FP-MRL's
nested-prefix loss directly rewards each coordinate for being useful in
isolation, so its magnitude ordering coincides with its information
ordering. Panels~(v)--(vi) confirm
this at the population level: the rank correlation between
$\mathbb{E}\!\left[|z_k|\right]$ and 1D probe accuracy is
$\rho\!=\!0.956$ for FP-MRL but only $\rho\!=\!0.850$ for MD-$\ell_1$,
quantifying how much more faithfully magnitude serves as a proxy for
informativeness under the nested-prefix objective. Together, these panels show that FP-MRL does not merely produce an ordered
representation but an ordered representation whose magnitude axis is
itself the information axis---an alignment that MD-$\ell_1$'s sparsity
prior approximates only loosely, and that constitutes the defining
specificity of the MRL objective.


\section{Conclusion}
\label{sec:conclusion}

We show that MRL induces an \emph{objective-specific} privileged basis, ordering dimensions via cumulative task gradients rather than data statistics or embedding mass. For LAE, FP-MRL recovers ordered PCA directions at no extra asymptotic cost; with nonlinear encoders, this structure incurs an $O(d)$ per-step overhead. Important future directions include characterizing whether such privileged bases persist in more complex nonlinear architectures, and alternative objectives such as contrastive or generative learning. Another important direction is to explore whether specific architectures are better suited to fully extract and utilize this task-aligned ordering. While these embeddings are optimal for the training objective by construction, natural next steps include understanding how they generalize to unseen data and downstream tasks, and whether regularizer-based approaches can recover comparable objective-specific ordering.



\bibliography{example_paper}
\bibliographystyle{icml2026}

\newpage
\appendix

\section{Loss Expansions and Gradients}
\label{app:expansion}


\textbf{Setup. }We expand each loss for a common worked example with input dimension
$p = 5$ and latent dimension $d = 3$. Let $x \in \RR^p$ be a single input. The encoder $B \in \RR^{d \times p}$
has rows $b_1, b_2, b_3$, and the decoder $A \in \RR^{p \times d}$ has
columns $a_1, a_2, a_3$. The rank-one contribution of dimension $k$
to the reconstruction is
\[
  y_k \;=\; a_k\, b_k^\top x \;\in\; \RR^p,
\]
that is, the latent activation $b_k^\top x \in \RR$ scaled by the
decoder column $a_k$. Stacking these contributions gives the
prefix-$m$ reconstruction and its residual:
\[
  \hat{x}^{(m)} \;=\; \sum_{k=1}^{m} y_k,
  \qquad
  r^{(m)} \;=\; x - \hat{x}^{(m)}.
\]
Each $\hat{x}^{(m)}$ is the model's reconstruction using only the
first $m$ latent dimensions, and $r^{(m)}$ measures what those $m$
dimensions fail to capture. The full reconstruction is
$\hat{x}^{(d)} = \sum_{k=1}^{d} y_k$ and the trailing residual is
$r^{(d)}$.
 
\subsection{LAE Loss}
\label{app:lae_expansion}

\paragraph{Definition.}
The standard linear autoencoder loss evaluates MSE at the full
reconstruction only:
\begin{equation}
  \LossAE(\theta)
  = \Bigl\|x - \sum_{k=1}^{d} y_k\Bigr\|^2
  = \|x - (y_1 + y_2 + y_3)\|^2.
\end{equation}
 
\paragraph{Expansion ($p=5,\,d=3$).}
\begin{align}
  \LossAE(\theta)
  &= \|x\|^2 - 2x^\top(y_1{+}y_2{+}y_3) \notag\\
  &\quad + \|y_1\|^2 + \|y_2\|^2 + \|y_3\|^2 \notag\\
  &\quad + 2y_1^\top y_2 + 2y_1^\top y_3 + 2y_2^\top y_3.
\label{eq:lae_expansion}
\end{align}
Every $y_k$ receives linear weight $1$ toward $x$ and every cross-term
$y_i^\top y_j$ has coefficient $2$, corresponding to weight matrix
$\mathbf{1}\mathbf{1}^\top \otimes I_p$.
 
 
\subsection{Full-Prefix MRL Loss}
\label{app:FP-MRL_expansion}
 
\paragraph{Definition.}
The FP-MRL loss sums MSE at every prefix reconstruction under $\omega_m = 1$:
\begin{align}
  \LossFPMRL(\theta)
  &= \underbrace{\|x - y_1\|^2}_{m=1}
   + \underbrace{\|x - (y_1{+}y_2)\|^2}_{m=2} \notag\\
  &\quad + \underbrace{\|x - (y_1{+}y_2{+}y_3)\|^2}_{m=3}.
\end{align}
 
\paragraph{Expanding each term.}
\begin{align}
  \|x - y_1\|^2
    &= \|x\|^2 - 2x^\top y_1 + \|y_1\|^2, \\
  \|x - (y_1{+}y_2)\|^2
    &= \|x\|^2 - 2x^\top(y_1{+}y_2) \notag\\
    &\quad + \|y_1\|^2 + 2y_1^\top y_2 + \|y_2\|^2, \\
  \|x{-}(y_1{+}y_2{+}y_3)\|^2
    &= \|x\|^2 - 2x^\top(y_1{+}y_2{+}y_3) \notag\\
    &\quad + \|y_1\|^2 + \|y_2\|^2 + \|y_3\|^2 \notag\\
    &\quad + 2y_1^\top y_2 + 2y_1^\top y_3 + 2y_2^\top y_3.
\end{align}
 
\paragraph{Summing and collecting.}
\begin{align}
  \LossFPMRL(\theta)
  &= 3\|x\|^2 \notag\\
  &\quad - \underbrace{2(3\,x^\top y_1 + 2\,x^\top y_2 + x^\top y_3)}_{\text{(1) ordered linear weights}} \notag\\
  &\quad + \underbrace{3\|y_1\|^2 + 2\|y_2\|^2 + \|y_3\|^2}_{\text{(2) ordered quadratic weights}} \notag\\
  &\quad + \underbrace{4\,y_1^\top y_2 + 2\,y_1^\top y_3 + 2\,y_2^\top y_3}_{\text{(3) ordered cross-dim.\ coupling}}.
  \label{eq:expanded_full}
\end{align}


The quadratic part is governed by
$S_3 \otimes I_5$, where 
\begin{equation}
  S_3 = \begin{pmatrix} 3 & 2 & 1 \\ 2 & 2 & 1 \\ 1 & 1 & 1
  \end{pmatrix},
  \quad (S_3)_{ij} = d - \max(i,j) + 1.
\end{equation}
 $(S_d)_{ij}$ counts the number of prefixes containing both
dimensions $i$ and $j$: a prefix of size $m$ contains $y_k$ iff
$k \le m$, so the count is $\sum_{m=\max(i,j)}^d 1 = d - \max(i,j) + 1$.
 
\subsection{Non-Uniform \texorpdfstring{$\ell_2$}{l2} Regularization}
\label{app:NU-L2_expansion}

\paragraph{Definition}
\begin{equation}
  \LossNUL(\theta)
  = \frac{1}{n}\|X - ABX\|_F^2
  + \|\Lambda^{1/2} B\|_F^2
  + \|A \Lambda^{1/2}\|_F^2,
\end{equation}
with $\Lambda = \mathrm{diag}(\lambda_1,\ldots,\lambda_d)$
\citep{kunin2019loss,bao2020regularized}.  Let
$\sigma_1^2 > \cdots > \sigma_d^2 > 0$ denote the top-$d$ eigenvalues of
the data covariance $\tfrac{1}{n}XX^{\top}$, with corresponding PCA
eigenvectors $u_1,\ldots,u_d$.

For the recovery guarantee to apply, the diagonal entries of $\Lambda$
must satisfy
\begin{equation}
  \label{eq:bao_constraints}
  0 \;<\; \lambda_1 \;<\; \lambda_2 \;<\; \cdots \;<\; \lambda_d
  \;<\; \sigma_d^{\,2}.
\end{equation}

Strict ordering ($\lambda_1 < \cdots < \lambda_d$):
    breaks the orthogonal symmetry of the uniform
    $\ell_2$ loss down to sign
    flips, making the set of global minima a discrete set of axis-aligned
    solutions rather than a continuous rotation orbit. Upper bound ($\lambda_d < \sigma_d^{\,2}$):
    ensures the shrinkage factors
    $(I - \Lambda S^{-2})^{1/2}$ in Theorem~2 in~\citet{bao2020regularized} are real and positive,
    so every latent dimension $k$ receives a non-trivial contribution
    from its assigned principal component.
 
\paragraph{Expansion ($p=5,\,d=3$).}
Writing $b_k \in \RR^p$ for the $k$-th row of $B$ and
$a_k \in \RR^p$ for the $k$-th column of $A$:
\begin{align}
  \LossNUL(\theta)
  &= \frac{1}{n}\|X - ABX\|_F^2 \notag\\
  &\quad + \lambda_1\|b_1\|^2 + \lambda_2\|b_2\|^2
           + \lambda_3\|b_3\|^2 \notag\\
  &\quad + \lambda_1\|a_1\|^2 + \lambda_2\|a_2\|^2
           + \lambda_3\|a_3\|^2.
\end{align}
The reconstruction term is identical across all dimensions; ordering
comes entirely from the scalar multipliers $\lambda_k$ on weight norms.

\subsection{Monotone-Decay \texorpdfstring{$\ell_1$}{l1} Loss}
 \label{app:MD-L1_expansion}
 
\paragraph{Definition.} We penalize the activation of each latent dimension $z_k = b_k^\top x$
with strictly increasing $\ell_1$ coefficients --- heavier on late
dimensions to discourage them from being used with $\alpha > 0$:
\begin{equation}
  \LossMDL(\theta)
  = \frac{1}{n}\|X - ABX\|_F^2
  + \frac{\alpha}{n} \sum_{k=1}^{d} k \sum_{s=1}^{n} |b_k^\top x_s|,
\end{equation}

\paragraph{Expansion ($p=5,\,d=3$).}
\begin{align}
  \LossMDL(\theta)
  &= \frac{1}{n}\|X - ABX\|_F^2 \notag \\
  &\quad + \frac{\alpha}{n}\Bigl(
      1\cdot{\textstyle\sum_s}|b_1^\top x_s|
    + 2\cdot{\textstyle\sum_s}|b_2^\top x_s| \notag \\
  &\hphantom{\quad + \frac{\alpha}{n}\Bigl(}
    + 3\cdot{\textstyle\sum_s}|b_3^\top x_s|
    \Bigr).
\end{align}
The penalty on $z_k$ grows with $k$, pushing the model to concentrate
information in early dimensions.
 
\subsection{Fisher discriminant objective}
\label{app:fisher_loss}

\paragraph{Definition.}
We are given data $\{(x_s, \tilde y_s)\}_{s=1}^{n}$ with
$x_s \in \RR^p$ and class labels $\tilde y_s \in \{1, \ldots, c\}$.
Let $n_c$ denote the number of samples in class $c$. We define the
class means $\mu_c$, global mean $\mu$, and
between- and within-class scatter matrices $\SBC$ and $\SWC$:
\begin{align}
  \mu_c &= \frac{1}{n_c} \sum_{s:\,\tilde y_s = c} x_s, \\
  \mu &= \frac{1}{n} \sum_{s=1}^{n} x_s, \\
  \SBC &= \sum_c \frac{n_c}{n} (\mu_c - \mu)(\mu_c - \mu)^\top, \\
  \SWC &= \sum_c \frac{1}{n} \sum_{s:\,\tilde y_s = c}
    (x_s - \mu_c)(x_s - \mu_c)^\top.
\end{align}
We seek directions $v \in \RR^p$ along which classes are maximally
separated relative to their internal spread, i.e.\ that maximize the
Fisher ratio
\begin{equation}
  J(v) \;=\; \frac{v^\top \SBC\, v}{v^\top \SWC\, v}.
\end{equation}
We assume $\SWC \succ 0$. Setting $\nabla_v J = 0$, we find that the stationary points
of $J$ are exactly the solutions of the generalized eigenproblem
\begin{equation}
  \SBC\, v_k \;=\; \gamma_k\, \SWC\, v_k,
  \qquad
  \gamma_1 \geq \gamma_2 \geq \cdots \geq \gamma_p \geq 0,
\end{equation}
with $J(v_k) = \gamma_k$. The leading direction $v_1$ maximizes the
Fisher ratio; the top-$d$ directions $v_1, \ldots, v_d$ jointly
maximize the trace-ratio objective
$\mathrm{Tr}\!\left[(V^\top \SWC V)^{-1} V^\top \SBC V\right]$ over
$V \in \RR^{p \times d}$, with optimal value
$\sum_{k=1}^{d} \gamma_k$. 
We note that, since $\SBC$ is a sum of $c$ rank-one terms whose
deviations $\mu_c - \mu$ sum to zero (weighted by $n_c/n$), we have
$\mathrm{rank}(\SBC) \leq c - 1$, so at most $c - 1$ of the
generalized eigenvalues $\gamma_k$ are nonzero. LDA thus yields
at most $c-1$ informative directions.

\paragraph{Training.}
We train the encoder $B \in \RR^{d \times p}$ using a supervised Fisher objective.
Let $z_s = B x_s \in \RR^d$ be the embedding of input $x_s$.
Define between-class and within-class scatter matrices in input space:
\begin{align}
  \SBC^{(x)} &= \sum_c \frac{n_c}{n}(\mu_c - \mu)(\mu_c - \mu)^\top, \\
  \SWC^{(x)} &= \sum_c \frac{1}{n}\sum_{s:\,\tilde{y}_s = c}(x_s - \mu_c)(x_s - \mu_c)^\top,
\end{align}
where $\mu_c$ and $\mu$ are class and global means. The corresponding scatter matrices in embedding space are
\begin{equation}
  \SBC^{(z)} = B \SBC^{(x)} B^\top, \qquad
  \SWC^{(z)} = B \SWC^{(x)} B^\top.
\end{equation}

We optimize the Fisher objective:
\begin{equation}
  \cL_{\mathrm{Fisher}}(B)
  = -\,\mathrm{Tr}\!\left[
  \left(\SWC^{(z)} + \varepsilon I\right)^{-1} \SBC^{(z)}
  \right],
  \label{eq:fisher}
\end{equation}
with $\varepsilon > 0$ for numerical stability. This objective maximizes class separation in the learned embedding by enlarging between-class variance relative to within-class variance.
At optimum, the row space of $B$ spans the top-$d$ LDA directions.

\paragraph{Analogy with PCA and FP-MRL.} The role of the Eckart--Young--Mirsky theorem in the PCA case is
played here by the Rayleigh--Ritz characterization of the trace-ratio
problem: each prefix loss is minimized iff
$\mathrm{rowspan}(B_{1:m,:})$ equals the span of the top-$m$ LDA
directions $v_1, \ldots, v_m$, defined as the leading solutions of
$\SBC^{(x)} v = \gamma\, \SWC^{(x)} v$ and ordered by descending
Fisher ratio $\gamma$. Summing over $m$ forces the nested chain
$\mathrm{span}(v_1) \subset \cdots \subset
\mathrm{span}(v_1, \ldots, v_d)$, in parallel to the PCA recovery of
\cref{thm:pca}, but with the principal directions $u_k$ of $XX^\top$
replaced by the LDA directions $v_k$ of the pair
$(\SBC^{(x)}, \SWC^{(x)})$. Unlike the MSE loss, $\cL_{\mathrm{Fisher}}$ is non-convex in $B$. In practice this manifests as
slower convergence and higher sensitivity to initialization and
$\varepsilon$ than the MSE case. 

\section{Efficient PCA recovery by FP-MRL}
\label{app:pca-proof}

This section gives the formal statement and proof of the PCA recovery
result sketched in \cref{sec:structure}, together with the computational
complexity analysis.

\begin{theorem}[Full-prefix linear MRL and PCA recovery]
  \label{thm:pca}
  Let $\Sigma = XX^\top$ have distinct top-$d$ eigenvalues
  $\sigma_1^2 > \cdots > \sigma_d^2 > 0$ with eigenvectors
  $u_1,\ldots,u_d$.
  \begin{enumerate}
  \item \emph{(Subspace recovery)}: every global optimum $(A^*, B^*)$ of
$\LossFPMRL$ satisfies
$\mathrm{colspan}(A^*_{1:m}) = \mathrm{span}(u_1, \ldots, u_m)$
for all $m \in [d]$, with each column $a_k^*$ aligned with $\pm u_k$.
The full-prefix structure pins the nested flag of subspaces exactly,
leaving only the discrete sign-flip symmetry $\{\pm 1\}^d$.
    \item \emph{(Full PCA recovery)}: under $A^\top A = I_d$, the
          unique global optimum is $A^* = [u_1,\ldots,u_d]$.
  \end{enumerate}
\end{theorem}

\begin{proof}
\textbf{Part (1).}
By the Eckart--Young--Mirsky theorem
\citep{eckart1936approximation,mirsky1960symmetric}, for each $m$:
\[
  \|X - A_{1:m}B_{1:m}X\|_F^2 \;\geq\;
  \sum_{k=m+1}^{p}\sigma_k^2,
\]
with equality iff
$\mathrm{colspan}(A_{1:m}) = \mathrm{span}(u_1,\ldots,u_m)$.
Summing over $m = 1,\ldots,d$, each eigenvalue $\sigma_k^2$ for
$k \leq d$ appears in the bound of every prefix $m < k$, contributing
coefficient $(k-1)$:
\[
  \LossFPMRL(A,B) \;\geq\;
  \sum_{k=2}^{d}(k-1)\sigma_k^2 + d\sum_{k=d+1}^{p}\sigma_k^2.
\]
Equality holds simultaneously at every $m$ iff
$\mathrm{colspan}(A_{1:m}) = \mathrm{span}(u_1,\ldots,u_m)$ for all
$m \in [d]$, which is achieved by the unique nested chain
$\mathrm{span}(u_1) \subset \cdots \subset
\mathrm{span}(u_1,\ldots,u_d)$.
Within each prefix, $A_{1:m}^*$ is free up to $T \in \mathrm{SO}(m)$
since the MSE loss depends only on the column space
\citep{baldi1989neural}.

\textbf{Part (2).}
Under $A^\top A = I_d$, distinct eigenvalues yield a unique orthonormal
minimizer at each scale, collapsing the $\mathrm{SO}(m)$ gauge to the
identity and hence $A^*_{1:m} = [u_1,\ldots,u_m]$ for all $m$
\citep{rippel2014learning}.
\end{proof}

\textbf{Computation complexity. }Naively, full-prefix MRL requires evaluating $d$ reconstruction losses,
one per prefix width, suggesting a $d$-fold cost over standard training.
In the linear setting this overhead is avoidable: every prefix loss is
determined by the same three $d \times d$ summaries of the data, encoder,
and decoder, computed once per forward pass.
The result is that full-prefix MRL has the same $O(npd)$ asymptotic cost
as a single standard reconstruction pass, with only $O(d^2)$ additive
overhead for prefix accumulation.
\begin{proof}
Let $Z = BX \in \RR^{d \times n}$. Write $a_i \in \RR^p$ for the $i$-th
column of $A$ and $z_i^\top \in \RR^{1 \times n}$ for the $i$-th row of
$Z$ (so $z_i \in \RR^n$). The prefix-$m$ reconstruction decomposes as
\[
  \hat{X}^{(m)} = A_{:,1:m}\, Z_{1:m,:} = \sum_{i=1}^{m} a_i z_i^\top.
\]. 
Expanding the Frobenius norm and substituting,
\[
  \ell_m := \|X - \hat{X}^{(m)}\|_F^2
         = \|X\|_F^2
           - 2\sum_{i=1}^m H_{ii}
           + \sum_{i,j=1}^m G^A_{ij}\, G^Z_{ij},
\]
where $Z = BX$, $H = A^\top X Z^\top$, $G^A = A^\top A$, and $G^Z = Z Z^\top$ are
$d \times d$ matrices, with total cost $O(npd)$ to form. They depend
on $m$ only through prefix truncation.
Forming $Z = BX$ and $X Z^\top$ costs $O(npd)$ in total. The Gram
matrices $G^A$ and $G^Z$ contribute $O(pd^2)$ and $O(nd^2)$, which are
absorbed when $n, p \gg d$. Once these are in hand, evaluating $\ell_m$ for all $m \in [d]$ then reduces to incremental
sums over $H$, $G^A$, and $G^Z$, taking $O(d^2)$ more work — total
cost $O(npd)$. Total cost: $O(npd + d^2) = O(npd)$.
\end{proof}


\section{Experimental Setup}
\label{app:exp_setup}

\begin{table*}[t]
  \centering
  \small
  \caption{\textbf{PCA/LDA recovery experiments: architecture and training parameters per dataset and loss
    family.} $p$: input dimensionality; $C$: number of classes;
    $n_{\mathrm{lda}}=C-1$: rank of the LDA subspace; $d$: latent
    dimensionality; $\cM$: S-MRL prefix set; $\sigma_d^2$: the
    $d$-th PCA eigenvalue of the training data.
    Optimiser is Adam with learning rate $10^{-3}$ and weight decay
    $10^{-4}$ in every row. }
  \label{tab:arch_dims}
  \setlength{\tabcolsep}{4pt}
  \begin{tabular}{llcccccccl}
    \toprule
    Dataset & Family & $p$ & $C$ & $n_{\mathrm{lda}}$ & $d$ & $\cM \text{ (S-MRL)}$
      & Batch &  Regulariser \\
    \midrule
    Synthetic data
      & MSE    & 69  & 20 & 19 & 50 & $\{5,10,25,50\}$ & 256
          & 
          $\alpha{=}0.01,\  \lambda_k\!\in\![\tfrac{\sigma_d^2}{18},\tfrac{\sigma_d^2}{2};\Delta\lambda{=}\tfrac{\lambda_{\max}-\lambda_{\min}}{(d-1)}]$ \\
    Synthetic data
      & Fisher & 69  & 20 & 19 & 19 & $\{5,10,15,19\}$ & full
         & $\varepsilon{=}10^{-4}$ \\
    Fashion-MNIST
      & MSE    & 784 & 10 & 9  & 32 & $\{4,8,16,32\}$  & 256
       & $\alpha{=}0.01,\  \lambda_k\!\in\![\tfrac{\sigma_d^2}{18},\tfrac{\sigma_d^2}{2};\Delta\lambda{=}\tfrac{\lambda_{\max}-\lambda_{\min}}{(d-1)}]$ \\
    Fashion-MNIST
      & Fisher & 784 & 10 & 9  & 9  & $\{2,4,7,9\}$    & full
       & $\varepsilon{=}10^{-4}$ \\
    \bottomrule
  \end{tabular}
  \\[3pt]
  \raggedright
\end{table*}

 \subsection{Model architecture and training dimensions}
\label{app:arch_dims}

\textbf{PCA/LDA recovery.} All models share an identical linear
autoencoder backbone for PCA/LDA recovery experiments.
\begin{equation}
  z = B x, \qquad \hat{x} = A z,
  \qquad
  B \in \RR^{d\times p},\ A \in \RR^{p\times d},
\end{equation}
where $x \in \RR^{p}$, $z \in \RR^{d}$, $B$ is the encoder and $A$
is the decoder. Both maps are bias-free and contain no
nonlinearity, so the end-to-end map $\hat{x}=ABx$ is strictly linear.
The values of the input dimension $p$ of data $x$, the number of
classes $C$, the latent dimension $d$, the S-MRL prefix set $\cM$,
and the optimiser/regulariser settings differ by dataset and loss
family; they are given row-by-row in Table~\ref{tab:arch_dims}. The decoder is
initialised with orthonormal columns ($A^\top A = I_d$) via reduced
QR decomposition.  When the decoder-orthogonality constraint is active
(\textsc{FP-MRL}), a QR re-projection of $A$ is applied after every
optimiser step.  Fisher-family losses use the same backbone but
ignore $A$ during training; labels enter only through the between-
and within-class scatter matrices $\SBC$ and $\SWC$ computed on
$z$, whose rank is bounded by $n_{\mathrm{lda}}=C-1$ where $C$ is
the number of classes.  

We use two datasets in the configurations of
Table~\ref{tab:arch_dims}.
Synthetic dataset: $n_{\mathrm{tot}}{=}10{,}000$
samples (500 per class), split stratified $70/10/20$ into $7{,}000$
train, $1{,}000$ val, and $2{,}000$ test. Details of synthetic data are given in appendix~\cref{app:synthetic}.
Fashion-MNIST dataset: the full
$n_{\mathrm{tot}}{=}70{,}000$ images, split stratified $70/10/20$
into $49{,}000$ train, $7{,}000$ val, and $14{,}000$ test. In both
cases the inputs are mean-centred on the training split (no variance
scaling, so the noise/signal eigenvalue contrast in the synthetic
data is preserved).

\textbf{Classification.} The classification experiment replaces the
linear autoencoder with a nonlinear encoder followed by a linear
classifier, trained end-to-end on MNIST
($p=784$, $C=10$) with a bottleneck width of $d=16$.  The
encoder
\begin{equation}
  z \;=\; \frac{f_\theta(x)}{\lVert f_\theta(x) \rVert_2},
  \qquad
  f_\theta:\RR^{p}\!\to\!\RR^{d},
\end{equation}
is a two-hidden-layer MLP
($p \!\to\! 256 \!\to\! 256 \!\to\! d$) with batch normalisation,
ReLU, and dropout ($0.1$) between layers, and an $\ell_2$-normalised
output.  A single linear classifier $W\!\in\!\RR^{C\times d}$ maps
$z$ to class logits.  We train two loss families on this shared
backbone: FP-MRL, and (MD-$\ell_1$). All models are optimised with Adam (learning rate $10^{-3}$,
weight decay $10^{-4}$) at batch size $128$ for up to $20$ epochs
with early stopping (patience $5$) on validation loss.  The best
checkpoint by validation loss is retained for evaluation. Unlike the linear
autoencoder setting, the nonlinear encoder breaks the shared
sufficient-statistics structure that made FP-MRL asymptotically free:
each of the $d$ prefix losses now requires a separate forward pass through
the classifier head, so FP-MRL incurs an $O(d)$ overhead per step relative
to MD-$\ell_1$.

\textbf{Prefix-wise evaluation.}  We report two test-set metrics
for every trained model.
\begin{itemize}
\item \emph{1-D probing accuracy.}  For each prefix
      $k=1,\dots,d$ we fit an independent $1$-D logistic regression
      on $z_{\text{train}}[:,k]$ (restricted to a $2{,}000$-sample
      subset for tractability) and score it on
      $z_{\text{test}}[:,k]$.  The resulting per-dimension accuracy
      profile quantifies how much class-discriminative information
      each coordinate carries in isolation, independently of any
      interaction with other dimensions.
\item \emph{Linear accuracy.}  For each prefix length
      $k=1,\dots,d$ we re-use the learned classifier:
      \begin{equation}
        \mathrm{logits}_k
        \;=\;
        z_{\text{test}}[:,1{:}k]\,W_{:,1{:}k}^{\top} \;+\; b,
      \end{equation}
      and report top-$1$ accuracy against the test labels.  This metric
      measures how well the trained classifier itself
      respects the prefix structure — i.e., whether the leading
      coordinates alone are sufficient for $W$ to discriminate the
      classes.
\end{itemize}

\begin{figure*}[!b]
    \centering
    \myfigure[width=\textwidth]{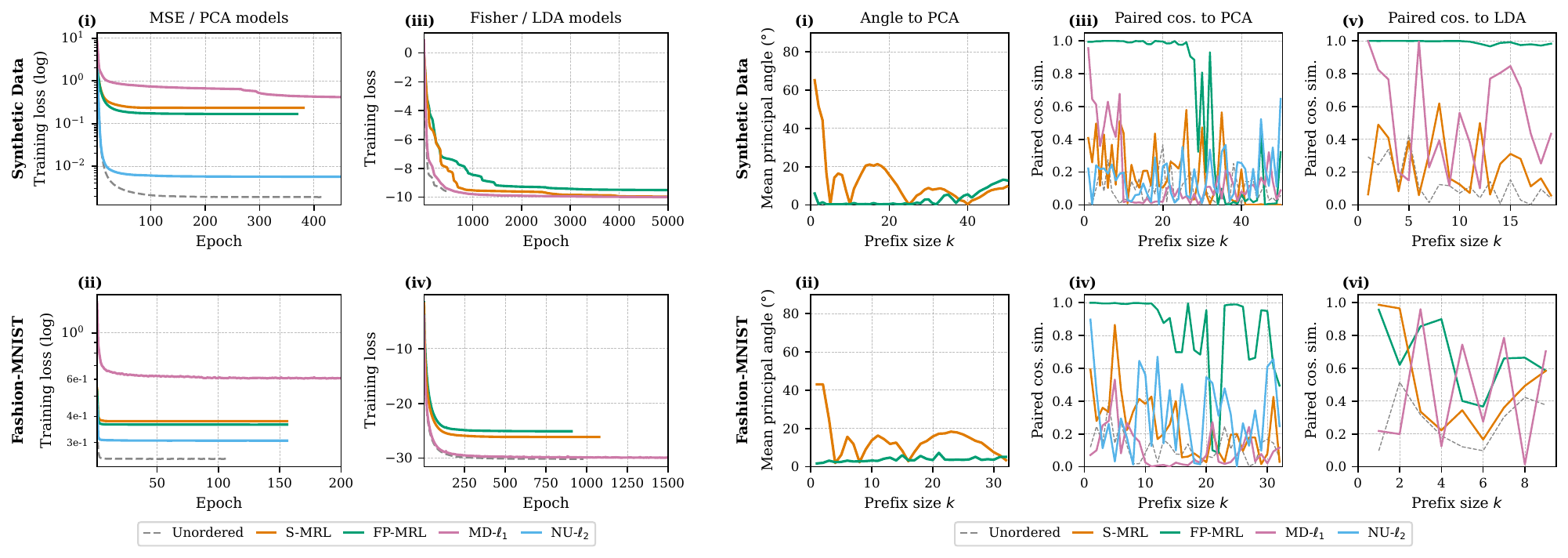}
    \begin{minipage}[t]{0.49\textwidth}
        \captionof{figure}{\textbf{Fisher/LDA optimization is harder to reach the global
optimum under prefix-based losses.} Training loss curves for MSE/PCA
models (i,~ii) and Fisher/LDA models (iii,~iv), on synthetic data (top)
and Fashion-MNIST (bottom). All methods reach the global optimum on the
LAE objective for both datasets, and on the LDA objective for synthetic
data. On Fashion-MNIST LDA, however, S-MRL and FP-MRL plateau above the
optimum reached by MD-$\ell_1$, NU-$\ell_2$, and the unordered
baseline, indicating that with the simple linear model architecture described in~\cref{app:arch_dims}, prefix-based losses struggle to
optimize the Fisher objective.}
        \label{fig:losses}
    \end{minipage}%
    \hfill
    \begin{minipage}[t]{0.49\textwidth}
        \captionof{figure}{\textbf{Basis alignment across symmetry-breaking mechanisms.}
        Mean principal angle between the leading $k$ learned prefix
        directions and the leading $k$ PCA eigenvectors (i,~ii), and
        paired cosine similarity between the $k$-th learned prefix
        direction and the $k$-th PCA (iii,~iv) or LDA (v,~vi)
        eigenvector, on synthetic data (top row) and Fashion-MNIST
        (bottom row). In (i,~ii), FP-MRL tracks the ordered PCA basis
        closely across all prefix sizes, while S-MRL aligns only at
        the discrete prefix sizes evaluated during training (see Table~\ref{tab:arch_dims}) and departs from the
        ordered basis at intermediate $k$. Across rest of the panels, MD-$\ell_1$,
        NU-$\ell_2$, and the unordered baseline recover the subspace
        but not the coordinate-wise ordering.}
        \label{fig:pairedCos}
    \end{minipage}
\end{figure*}

\subsection{Synthetic data generation}
\label{app:synthetic}

The synthetic dataset is designed so that the top PCA and LDA
subspaces are orthogonal by construction, providing a controlled setting
in which to test whether each loss recovers its objective-specific
subspace.  Each input $x \in \RR^p$ is partitioned into two disjoint
blocks,
\begin{equation}
  x \;=\;
  \bigl[\,\underbrace{x_{\text{noise}}}_{p_{\text{noise}}}\;\big|\;
          \underbrace{x_{\text{sig}}}_{p_{\text{sig}}}\,\bigr],
  \qquad
  p \;=\; p_{\text{noise}} + p_{\text{sig}},
\end{equation}
drawn independently per sample from
\begin{equation}
  x_{\text{noise}} \sim \mathcal{N}(0, \Sigma_{\text{noise}}),
  \qquad
  x_{\text{sig}} \mid c \sim \mathcal{N}(\mu_c, \tau_{\text{sig}}^2 I),
\end{equation}
where $c \in \{1, \ldots, C\}$ is the class label.
The noise block carries no class information and has large variance, so
PCA eigenvectors concentrate on it.  The signal block has small
within-class variance but class-dependent means, so LDA directions
concentrate on it.  Setting $p_{\text{sig}} = C-1$ ensures that
between-class scatter is full rank and that the signal block is exactly
spanned by the $C-1$ LDA discriminants.

We use $C = 20$ classes, $p_{\text{noise}} = 50$,
$p_{\text{sig}} = C - 1 = 19$ (so $p = 69$),
$\tau_{\text{noise}} = 5.0$, $\tau_{\text{sig}} = 0.1$, and base
class-mean scale $\beta_0 = 1.0$.  We draw $500$ samples per class
(total $n = 10{,}000$) and split $70/10/20$ into train/val/test
(stratified).  Features are mean-centred on the train split only; we
deliberately do not standardise variance, which would destroy the
variance contrast between the two blocks.

\paragraph{Noise block.}
$\Sigma_{\text{noise}} \in \RR^{p_{\text{noise}} \times p_{\text{noise}}}$
is diagonal with a geometrically decaying spectrum,
\begin{equation}
  (\Sigma_{\text{noise}})_{ii}
  \;=\; \tau_{\text{noise}}^2 \, \rho_{\text{noise}}^{\,2i},
  \qquad i = 0, \ldots, p_{\text{noise}} - 1,
\end{equation}
with $\rho_{\text{noise}} = 0.9$, so consecutive noise eigenvalues are
well-separated and individual PCA eigenvectors (not just the subspace)
are identifiable.

\paragraph{Signal block.}
Class means in the signal block are constructed to give a ground-truth
LDA ordering.  Let $H \in \RR^{C \times p_{\text{sig}}}$ have orthonormal
columns, each orthogonal to $\mathbf{1}_C$ (so the grand mean is zero on
every signal dimension), obtained by reduced QR of a Gaussian matrix
projected onto $\mathbf{1}_C^\perp$.  Let $\bar{\mu} \in \RR^{C \times p_{\text{sig}}}$
be the matrix whose $c$-th row is the class-$c$ signal-block mean $\mu_c$.
We set
\begin{equation}
  \bar{\mu} \;=\; H \, \mathrm{diag}(\beta),
  \qquad
  \beta_i \;=\; \beta_0 \, \rho_{\text{sig}}^{\,i},
\end{equation}
with $\beta_0 = 1.0$ and $\rho_{\text{sig}} = 0.7$.  The between-class
scatter in signal space is then
\[
  \SBC^{(\text{sig})} \;=\; \tfrac{1}{C}\,
  \mathrm{diag}\bigl(\beta_0^2, \ldots, \beta_{p_{\text{sig}}-1}^2\bigr),
\]
so the $i$-th LDA direction aligns with signal axis $i$ and class
separability decreases monotonically with $i$.

\subsection{Evaluation metrics}
\label{app:eval_metrics}

Let $\{q_1,\dots,q_d\}$ denote the learned prefix directions (unit-norm
columns of the decoder $A$, or rows of the encoder $B$ in the Fisher
case), and let $\{e_1,\dots,e_d\}$ denote the reference basis (PCA
eigenvectors $u_k$ of $\Sigma$, or LDA discriminants $v_k$), also
unit-norm and ordered by descending eigenvalue.  We use three metrics
to compare the learned and reference bases.

\paragraph{Paired cosine similarity.}
Measures whether the $k$-th learned direction matches the $k$-th
reference direction, capturing coordinate-wise ordering:
\begin{equation}
  \mathrm{PairedCos}(k) \;=\; \bigl|\langle q_k,\, e_k\rangle\bigr|,
  \qquad k \in [d].
\end{equation}

\paragraph{Max cosine similarity.}
Measures whether the $k$-th learned direction lies along \emph{any}
reference direction, capturing subspace membership without requiring
ordering:
\begin{equation}
  \mathrm{MaxCos}(k) \;=\; \max_{j \in [d]}
  \bigl|\langle q_k,\, e_j\rangle\bigr|.
\end{equation}

\paragraph{Mean principal angle.}
Measures alignment between the leading-$k$ learned subspace and the
leading-$k$ reference subspace, independent of any per-axis
correspondence.  Let $Q_k = [q_1,\dots,q_k]$ and $E_k = [e_1,\dots,e_k]$
(with orthonormalised columns), and let
$\eta_1 \geq \dots \geq \eta_k$ be the singular values of
$Q_k^\top E_k$.  The principal angles are
$\varphi_i = \arccos(\eta_i)$, and we report
\begin{equation}
  \overline{\varphi}(k) \;=\; \frac{1}{k}\sum_{i=1}^{k}\varphi_i.
\end{equation}
A value of $0^\circ$ indicates the two leading-$k$ subspaces coincide;
$90^\circ$ indicates they are orthogonal.

\begin{figure*}[t]
    \centering
    \begin{subfigure}{\textwidth}
        \centering
        \myfigure[width=0.85\textwidth]{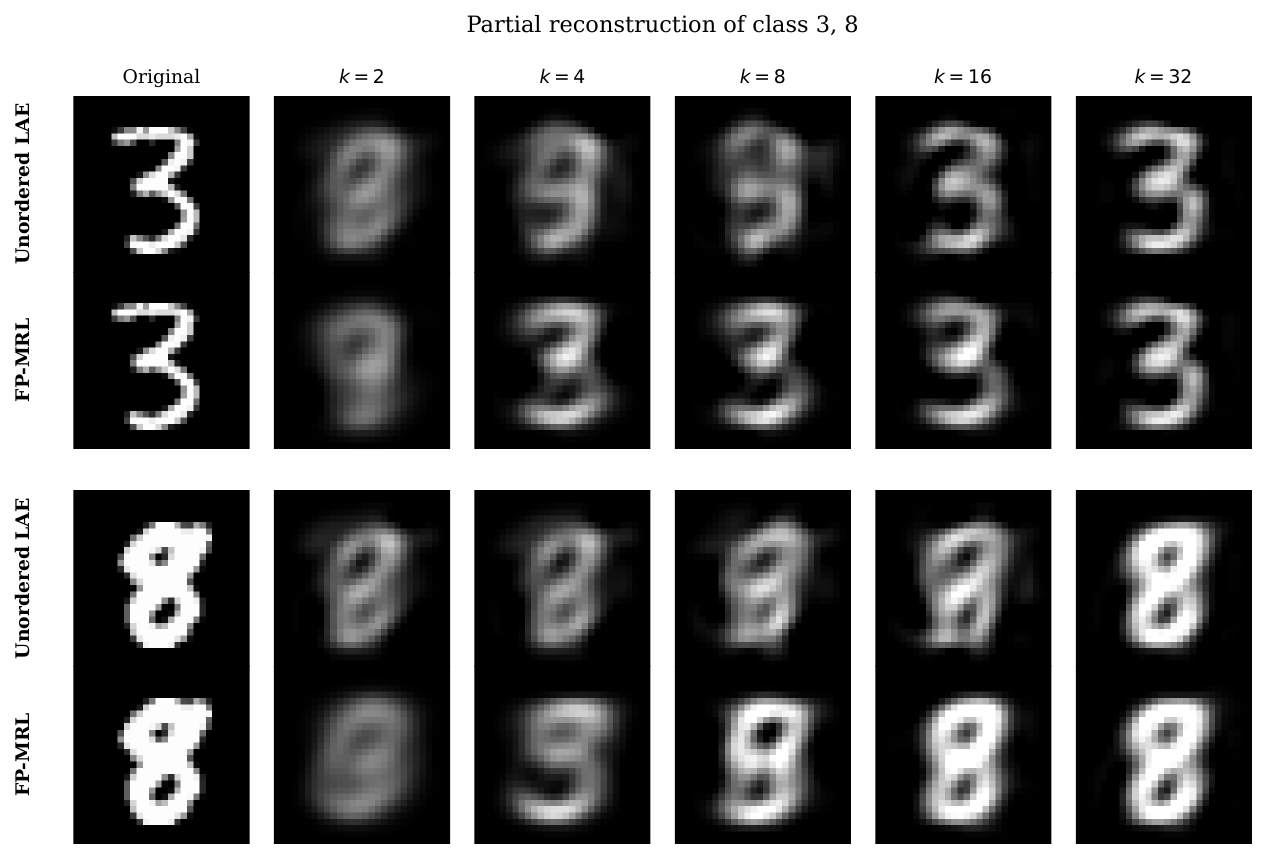}
        \caption{Reconstruction grid.}
        \label{fig:reconstruction}
    \end{subfigure}
    
    \vspace{0.5em}
    
    \begin{subfigure}{\textwidth}
        \centering
        \myfigure[width=0.85\textwidth]{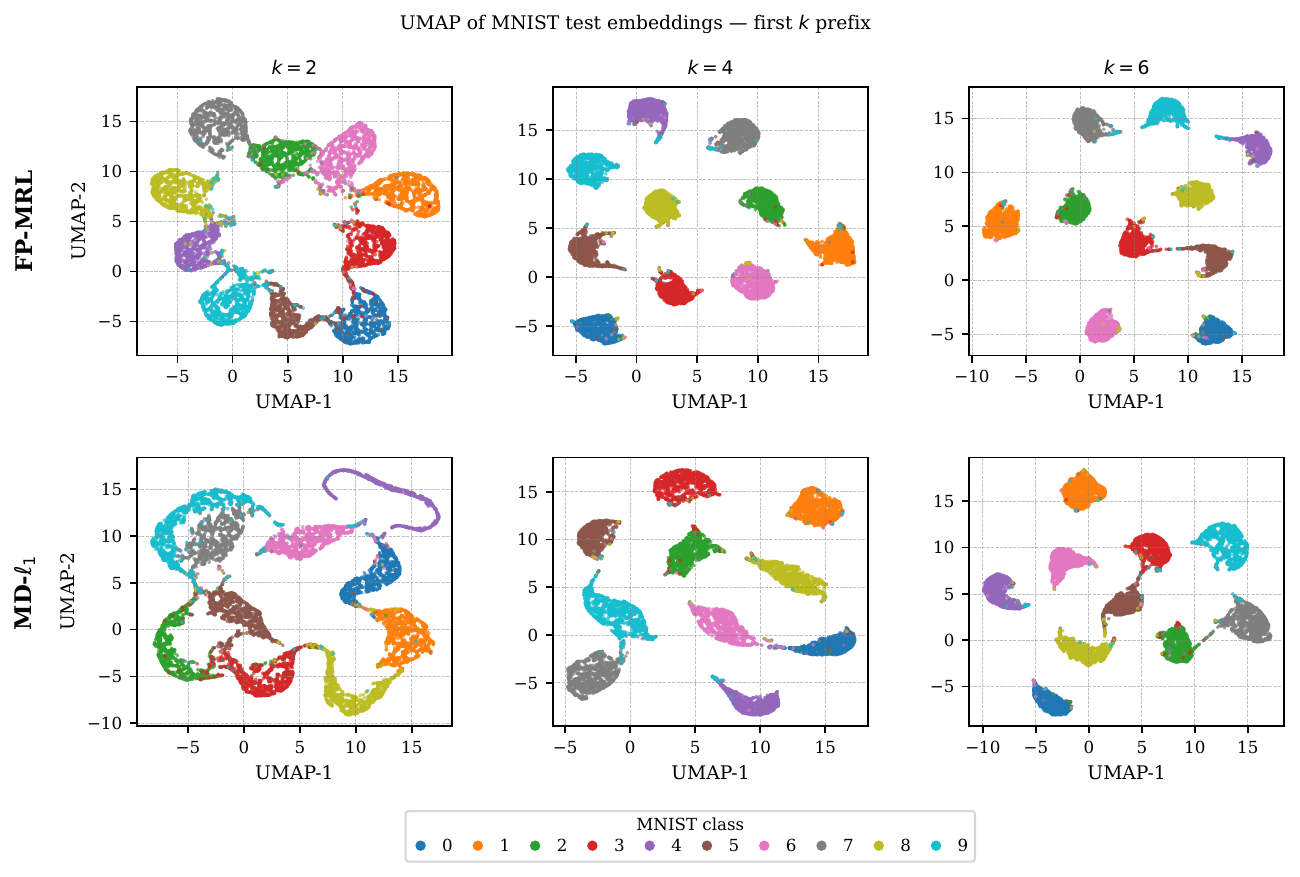}
        \caption{UMAP clustering.}
        \label{fig:cluster}
    \end{subfigure}
    
    \caption{
    \textbf{Prefix-truncated representation quality on MNIST.}
    \textbf{(a)} Partial reconstructions of digits 3 and 8 using only the
    first $k\!\in\!\{2,4,8,16,32\}$ latent coordinates. The unordered LAE
    baseline produces incoherent outputs at small $k$, while \textsc{FP-MRL}
    yields recognizable digits even from $k\!=\!2$, with quality improving
    monotonically as more coordinates are added.
    \textbf{(b)} UMAP projections of the first-$k$ prefix of the test-set
    embeddings for $k\!\in\!\{2,4,6\}$. \textsc{FP-MRL} (top) produces
    visibly more compact and class-separated clusters at every prefix
    length than the \textsc{MD-}$\ell_1$ baseline (bottom).
    }
    \label{fig:recon_and_cluster}
\end{figure*}


\begin{table*}[t]
  \centering
  \small
  \caption{Summary of notation used throughout the paper.}
  \label{tab:notation}
  \setlength{\tabcolsep}{4pt}
  \begin{tabular}{@{}ll@{\hspace{2em}}ll@{}}
    \toprule
    Symbol & Meaning & Symbol & Meaning \\
    \midrule
    \multicolumn{4}{@{}l}{\textit{Data and dimensions}} \\
    $p$ & Input dimensionality                       & $C$ & Number of classes \\
    $d$ & Embedding dimensionality                   & $L$ & Output dimensionality of task head \\
    $n$ & Number of training samples                 & $\Sigma = XX^\top$ & Empirical data covariance \\
    $X \in \RR^{p\times n}$ & Zero-centered data matrix & $\sigma_k^2,\, u_k$ & $k$-th PCA eigenvalue / eigenvector of $\Sigma$ \\
    $x \in \RR^p$ & Single input sample              & $\tilde{y}_s$ & Target / label for sample $s$ \\
    $s \in [n]$ & Sample index                       & $k \in [d]$ & Latent dimension index \\
    $i, j \in [d]$ & Dimension indices for matrix entries & & \\
    \midrule
    \multicolumn{4}{@{}l}{\textit{Model parameters}} \\
    $B \in \RR^{d\times p}$ & Linear encoder         & $a_k \in \RR^p$ & $k$-th column of decoder $A$ \\
    $A \in \RR^{p\times d}$ & Linear decoder; cols.\ $a_1,\ldots,a_d$ & $W^{(m)} \in \RR^{L\times m}$ & Linear task head at prefix scale $m$ \\
    $b_k \in \RR^p$ & $k$-th row of encoder $B$      &  &  \\
    \midrule
    \multicolumn{4}{@{}l}{\textit{Latent and reconstruction}} \\
    $z = Bx \in \RR^d$ & Embedding                 & $y_k = z_k a_k \in \RR^p$ & Rank-1 contribution of dim.\ $k$ \\
    $z_k = b_k^\top x$ & $k$-th prefix of the embedding    & $\hat{x}^{(m)} = \sum_{k=1}^m y_k$ & Prefix-$m$ reconstruction \\
    $z_{1:m}$ & First $m$ prefix of $z$         & $r^{(m)} = x - \hat{x}^{(m)}$ & Prefix-$m$ residual \\
    \midrule
    {\textit{Prefix structure}} \\
$\cM \subseteq [d]$ & Nesting set of prefix sizes
  & $\omega_m > 0$ & Per-prefix weight (we use $\omega_m{=}1$) \\
$m \in \cM$ & A prefix size
  & $w_k = \sum_{m=k}^d \omega_m$ & Cumulative weight on dimension $k$ \\
$Z = BX \in \RR^{d\times n}$ & Embedding matrix (all samples)
  & $\hat{X}^{(m)} = A_{:,1:m} Z_{1:m,:}$ & Prefix-$m$ reconstruction (matrix form) \\
$\ell_m = \|X - \hat{X}^{(m)}\|_F^2$ & Prefix-$m$ squared loss
  & $S_d^\omega$ & Prefix coupling matrix; $(S_d^\omega)_{ij}{=}w_{\max(i,j)}$ \\
 &  & $H,\, G^A,\, G^Z$ & Sufficient $d{\times}d$ statistics (App.~\ref{app:pca-proof}) \\
\midrule
  \multicolumn{4}{@{}l}{\textit{Symmetry and identifiability}} \\
    $\GL(d)$ & \multicolumn{3}{l}{Invertible reparameterizations of latent space (broader than $\mathrm{SO}(d)$)} \\
    $\mathrm{SO}(d)$ & Rotations: $T^\top T{=}I,\ \det T{=}+1$       & $\cG_\cL$ & Subgroup of $\mathrm{SO}(d)$ leaving $\cL$ invariant at $\theta^*$ \\
    $T \cdot (A, B)$ & Latent reparameterization $(AT,\, T^{-1}B)$   & $\cG_\cL = \{I\}$ & Full identifiability \\
    \midrule
    \multicolumn{4}{@{}l}{\textit{Losses and regularizers}} \\
    $\cL$ & Generic differentiable training loss        & 
    $\LossFisher$ & Fisher discriminant objective \\
    $\theta$ & All trainable parameters                  & $\Lambda = \mathrm{diag}(\lambda_1,\ldots,\lambda_d)$ & NU-$\ell_2$ coefficient matrix \\
    $\theta^*$ & Optimum (minimizer) of $\cL$            & $\alpha$ & MD-$\ell_1$ regularization strength \\
    $\LossAE$ & Standard LAE (full-width MSE) loss & $\varepsilon$ & Fisher numerical stabilizer \\
    $\LossMRL$ & Matryoshka loss (general $\cM$)   & $\LossNUL$ & Non-uniform $\ell_2$ loss (NU-$\ell_2$) \\
    $\LossFPMRL$ & Full-prefix MRL loss ($\cM = [d]$) & $\LossMDL$ & Monotone-decay $\ell_1$ loss (MD-$\ell_1$) \\
    \midrule
    \multicolumn{4}{@{}l}{\textit{Fisher / LDA}} \\
    $\mu_c,\, \mu$ & Class-$c$ mean / global mean    & $\SBC^{(z)},\, \SWC^{(z)}$ & Between- / within-class scatter (embedding) \\
    $\SBC^{(x)},\, \SWC^{(x)}$ & Between- / within-class scatter (input) & $v_k,\, \gamma_k$ & $k$-th LDA direction / Fisher ratio \\
    & & $n_{\mathrm{lda}} = C{-}1$ & Rank of LDA subspace \\
    \midrule
    \multicolumn{4}{@{}l}{\textit{Methods (abbreviations)}} \\
    LAE   & Linear autoencoder                       & FP-MRL      & Full-prefix MRL ($\cM = [d]$) \\
    MRL   & Matryoshka Representation Learning       & NU-$\ell_2$ & Non-uniform $\ell_2$ regularization \\
    S-MRL & Sparse MRL ($O(\log d)$ prefixes)        & MD-$\ell_1$ & Monotone-decay $\ell_1$ regularization \\
          &                                          & PCA / LDA   & Principal / Linear Discriminant Analysis \\
    \midrule
    \multicolumn{4}{@{}l}{\textit{Evaluation metrics}} \\
    $q_k \in \RR^p$ & $k$-th learned prefix direction (unit norm)
      & $Q_k = [q_1,\dots,q_k]$ & Leading-$k$ learned basis matrix \\
    $e_k \in \RR^p$ & $k$-th reference basis direction (PCA $u_k$ / LDA $v_k$)
      & $E_k = [e_1,\dots,e_k]$ & Leading-$k$ reference basis matrix \\
    $\mathrm{PairedCos}(k)$ & Paired cosine sim.\ $|\langle q_k, e_k\rangle|$
      & $\eta_i$ & $i$-th singular value of $Q_k^\top E_k$ \\
    $\mathrm{MaxCos}(k)$ & Max cosine sim.\ $\max_j |\langle q_k, e_j\rangle|$
      & $\varphi_i = \arccos(\eta_i)$ & $i$-th principal angle \\
    & & $\overline{\varphi}(k)$ & Mean principal angle over leading-$k$ subspaces \\
    \bottomrule
  \end{tabular}
\end{table*}

\end{document}